\documentclass[11pt]{article}
\usepackage[margin=1in]{geometry}

\usepackage{amsthm,amssymb,amsmath,authblk,titlefoot}
\usepackage{color}
\usepackage{graphicx}
\usepackage{subfig,caption,enumitem}

\usepackage[utf8]{inputenc}

\usepackage[round]{natbib}

\newtheorem{theorem}{Theorem}

\newtheorem{lemma}{Lemma}

\newtheorem{observation}{Observation}
\newtheorem{definition}{Definition}
\newtheorem{notation}{Notation}
\newtheorem{property}{Property}
\newcommand{\E}{\mathbb{E}}

\begin{document}

\date{\today}
\title{Size of Interventional Markov Equivalence Classes\\ in Random DAG Models}
\author[1]{Dmitiry Katz}
\author[1]{Karthikeyan Shanmugam}
\affil[1]{IBM Research, NY \& MIT-IBM Watson AI Lab.}
\author[2]{Chandler Squires}
\author[2]{Caroline Uhler}
\affil[2]{ MIT, Cambridge, MA.}

\maketitle
 \unmarkedfntext{\small \\ Author email addresses \\ \texttt{dkatzrog@us.ibm.com}, \texttt{karthikeyan.shanmugam2@ibm.com},  
\texttt{csquires@mit.edu} \& \texttt{cuhler@mit.edu} }

\begin{abstract}
 Directed acyclic graph (DAG) models are popular for capturing causal relationships. % observable from data. Under conditional independence constraints, there are many eqiuivalent DAGs forming a \textit{markov equivalence class} (MEC). This class shrinks further upon incorporating information from interventional datasets. We investigate  size and various other properties of interventional equivalence classes of random order DAGs of constant edge density. A random order DAG is a random ordering of a Erd\H{o}s R\'{e}nyi graph.
 From observational and interventional data, a DAG model can only be determined up to its \emph{interventional Markov equivalence class} (I-MEC). We investigate the size of MECs for random DAG models generated by uniformly sampling and ordering an Erd\H{o}s-R\'{e}nyi graph. For constant density, we show that the expected $\log$ observational MEC  size asymptotically (in the number of vertices) approaches a constant. We characterize I-MEC size in a similar fashion in the above settings with high precision. %These results hold even for interventional markov equivalence classes after constant number of interventions designed by various natural classes of algorithm. 
 We show that the asymptotic expected number of interventions required to fully identify a DAG is a constant. These results are obtained by exploiting Meek rules and coupling arguments to provide sharp upper and lower bounds on the asymptotic quantities, which are then calculated numerically up to high precision.
 %We also show that the asymptotic expected number of interventions required to fully identify the DAG is constant. Our techniques, exploit properties of Meek rules and coupling arguments on random order DAGs. We provide very sharp computable upper and lower bounds on the asymptotic quantities. Then, we calculate these upper and lower bounds numerically upto very desirable precision results to establish these results. Our findings will have significant implications for performance of greedy search procedures that operate on the DAG space instead of over the space of MECs.
 Our results have important consequences for experimental design of interventions and the development of algorithms for causal inference.
\end{abstract}

\section{Introduction}
Directed acyclic graphs (DAGs) are popular models for capturing causal relationships among a set of variables. This approach has found important applications in various areas including biology, epidemiology and sociology \citep{gangl2010causal,lagani2016probabilistic}. A central problem in these applications is to learn the causal DAG from observations on the nodes. A popular approach is to infer missing edges based on conditional independence information that is learned from the data~\citep{Spirtes,kalisch2007estimating}. However, multiple DAGs can encode the same set of conditional independences. %from a data \cite{meek1995causal}.  Therefore, 
Hence in general the causal DAG can only be learned up to a \textit{Markov equivalence class} (MEC) and interventional data is needed in order to identify the causal DAG.

%Unfortunately, equivalence classes may contain a super exponential number of candidate DAGs \cite{gillispie2001enumerating}. To narrow down, interventional datasets are used when available and they narrow down the space of graphs due to the additional set of conditional independencies entailed by the interventional distributions. After interventions, the equivalence class shrinks further. 

While an MEC may contain a super exponential number of candidate DAGs, \cite{gillispie2001enumerating} showed by enumerating all MECs up to 10 nodes that for small graphs (up to 10 nodes) an MEC on average contains about four DAGs and that about a quarter of all MECs consist of a unique DAG. Generalizing these results to larger graphs is critical for estimating the average number of interventional experiments needed for identifying the underlying causal DAG. More generally, %\citet{hauser2012characterization} defined an \emph{interventional Markov equivalence class} (I-MEC) as a collection of DAGs that satisfy the same conditional independence relations given observational data and data from a particular set of interventional experiments. Similarly, 
given the recent rise in interventional data in genomics enabled by genome editing technologies~\citep{xiao2015gene}, it is of great interest to understand the average reduction in the size of MECs through the availability of interventional data, i.e., to characterize the average size of an \emph{interventional Markov equivalence class} (I-MEC). Further, such an analysis would also shed light on the number of additional interventions needed to uniquely identify the underlying causal DAG moving away from worst case bounds.

%If the data generating causal model corresponds to a given DAG, then it entails a set of conditional independencies on the observational data. However, multiple DAGs can map to the same set of observed conditional independencies from a data \cite{meek1995causal}.  Therefore, the causal DAG can only be learnt upto a \textit{markov equivalence class}(MEC). For some DAGs, an equivalence class may have super exponential number of candidate graphs \cite{gillispie2001enumerating}. To narrow down, interventional datasets are used when available and they narrow down the space of graphs due to the additional set of conditional independencies entailed by the interventional distributions. After interventions, the equivalence class shrinks further. 

The problem of characterizing the size of an MEC or I-MEC is not only of interest for experimental design of interventions but also from an algorithmic perspective. A popular approach to causal inference is given by score-based methods that assign a score such as the Bayesian Information Criterion (BIC) to each DAG or MEC and greedily optimize over the space of DAGs~\citep{castelo2003inclusion}, a combination of permutations and undirected graphs~\citep{teyssier2012ordering,raskutti2013learning,Solus,Mohammadi} or MECs~\citep{Meek_GES,brenner2013sparsityboost}. Similar score-based approaches have also been developed in the interventional setting~\citep{hauser2012characterization,wang2017permutation,Yang2018}. While a greedy step in the space of graphs can easily be defined (addition, removal or flipping of an edge), a greedy step in the space of Markov equivalence classes is complicated~\citep{Meek_GES}. Hence performing a greedy algorithm in the space of MECs only makes sense if the space of MECs is significantly smaller as compared to the space of DAGs. For instance showing that typically occurring MECs or I-MECs are small would imply that graph-based search procedures operate on a similar search space as the ones that use MECs, but can do so using simpler moves. %Further, the analysis would shed light on the number of additional interventions needed to learn a causal DAG moving away from worst case bounds.

Motivated by these considerations, in this work, we initiate the study of interventional and observational MECs for random DAG models. We focus on \textit{random order DAGs}, where the skeleton is a random Erd\H{o}s R\'{e}nyi graph with constant density $\rho$ and the ordering is a random permutation. We derive tight bounds for the asymptotic versions of various metrics on the I-MECs. More specifically, our contributions are as follows:
\begin{enumerate}[itemsep=-4pt]
\item We derive tight upper and lower bounds on  (a) the asymptotic expected number of unoriented edges in an I-MEC given data from $r=0,1,2 \ldots$ interventions; (b) the asymptotic probability that the I-MEC is a unique DAG given data from $r$ interventions; (c) the asymptotic number of additional interventions needed to fully discover the DAG given data from $r$ interventions; and (d) the asymptotic expected $\log$-size of the I-MEC given data from $r$ interventions.
\item We also provide tight bounds for the number of unoriented edges in the I-MEC when $r$ interventions have been performed using different algorithms for choosing the interventions given the observational MEC as input.
\item If $M(r)_n$ is the metric of interest of a random order DAG of size $n$ and $r \geq 0$ interventions, then our bounds are of the following form: $\E[M(r)_n] \leq \E[M(r)_{\infty}] \leq \E[M(r)_n] + \epsilon_n$. Here, $M(r)_{\infty}$ is the limiting asymptotic metric, which we show is well defined and exists. We also show that $\epsilon_n$ decays exponentially fast in $n$ for constant density $\rho$. 
\item We numerically compute $\E[M(r)_n]$ through Monte Carlo simulations for $n$ as large as $110$ at which point $\epsilon_n$ is a small constant for various parameter regimes.
\item One of the surprising results is that for constant density random order DAGs, all the above metrics tend asymptotically to a constant. Through a combination of analysis of our bounds and numerical computation, we can characterize these constants precisely.
\item As an example of the nature of our results, quite surprisingly, the asymptotic (as $n \rightarrow \infty$) expected $\log$-observational MEC size of a random order DAG with density $0.5$ is at most $3.497$ with probability at least $0.99$ (see Theorem \ref{thm-upp-bounds}).
\end{enumerate}

All omitted proofs can be found in the supplemental material.

\textbf{Related Work:} There is currently only limited work available on counting and characterizing MECs. In \citep{gillispie2001enumerating}, the authors enumerated all MECs on DAGs with $p \leq 10$ nodes and analyzed the total number of MECs, the average size of an MEC, and the proportion of MECs of size one on $p$ nodes. Motivated by this work,  \cite{gillispie2006formulas}, \cite{steinsky2003enumeration}, and \cite{wagner2013asymptotic} provided formulas for counting MECs of a specific size. Supplementing this line of work, \cite{he2016formulas} developed various methods for counting the size of a given MEC. Finally, \cite{radhakrishnan2016counting} addressed these enumerative questions using a pair of generating functions that encode the number and size of MECs for DAGs with a fixed skeleton (i.e.~underlying undirected graph) and also applied these results to derive bounds on the MECs for various families of DAGs on trees~\citep{radhakrishnan2018counting}. 

%analyzed MECs on a given undirected graph skeleton with a given number of immoralities or a given size by connecting with graph polynomials that actas generating functions. They use this framework to derive interesting results about MECs for various families of tree graphs in \cite{radhakrishnan2018counting}. 

Another line of work \citep{hu2014randomized,hauser2012two,shanmugam2015learning,eberhardt2012number,hyttinen2013experiment,kocaoglu2017cost} aims at characterizing the number of interventions required to learn a causal DAG completely. While some of these works deal with the active learning setting \citep{shanmugam2015learning,hauser2012two}, others choose interventions non-adaptively given the observational MEC \citep{hu2014randomized,eberhardt2012number,hyttinen2013experiment,kocaoglu2017cost,bello2017learning} and hence are concerned with the worst-case scenario.

\section{Preliminaries and Definitions}

In this work, we characterize the asymptotic behavior of different metrics that capture the amount of ``causal relationships" which can be inferred from observational and interventional data on random DAG models. %learned from data obtained from a Causal Bayesian Network whose structure is a random order DAG after $k=0,1,2 ...$ single node interventions. 
In this section, we describe the random orderDAG model, briefly review causal DAG models and Markov equivalence, and introduce the metrics that we will analyze in this work.
 
\subsection{Random Order DAG Model} \label{sec:orderDAG}

Let  $G = (V, E)$ be a directed acyclic graph (DAG) with vertices $V= [n]$ and directed edges $E \subseteq V \times V$.
% CHANDLER: previous definition of edges was weird
% $E=\{(i,j):~ \mathrm{ordered~pair~}(i,j)\mathrm{~ is~ a ~directed ~edge}\}$. %Define the adjacency matrix $A_{i,j}$ where $A_{i,j}=1$ if the ordered pair $(i,j) \in E$. 
% Let $\mathrm{Pa}(v)$ denote the parents and $\mathrm{Nd}(v)$ the non-descendents of the node $v$ in $G$ respectively. 
% CHANDLER: Don't use the above notation so don't define it
A random \textbf{orderDAG}  with density $\rho$ on $n$ vertices is a DAG $G_n$ whose \emph{skeleton} (i.e., underlying undirected graph) is given by an Erd\"os-R\'enyi graph on $n$ vertices with edge probability $\rho$ and whose edges are oriented according to a total ordering which is uniformly sampled among all permutations of $n$ vertices. % of $G_n$ which directs the edges of $G^\prime_n$. 
We denote a graph $G_n$ sampled from this model by $G_n \sim orderDAG(n, \rho)$.

\textbf{Remark:} Our sampling procedure is a standard one used for testing causal inference algorithms. It is for example used in the well known \texttt{pcalg} R package\footnote{https://rdrr.io/rforge/pcalg/man/randomDAG.html}. A different sampling scheme would be to sample DAGs uniformly at random from all DAGs in which isomorphic DAGs would not be double counted. However, such a sampling scheme is difficult to perform in practice, while ours has a generative model that is easy and intuitive. Limited prior computational evidence in the observational setting suggests that the two sampling schemes behave similarly~\citep{gillispie2001enumerating}. 

\subsection{Markov Equivalence}

%\begin{definition}
A joint distribution $P$ on the variables $(X_v)_{v\in V}$ associated to the vertices of a DAG $G$ is \emph{Markov} with respect to $G$ if for any node $v \in G$, $X_v$ is conditionally independent of its non-descendents given its parents. In this case we say that $P \in \mathcal{M}(G)$.  
%\end{definition}
Two directed acyclic graphs $G$ and $G'$ are in the same \textit{Markov equivalence class} (MEC) if and only if $\mathcal{M}(G) = \mathcal{M}(G')$.
%\begin{definition}
% A graph $G$ entails a conditional independency of the form ``$X$ and $Y$ are conditionally independent given $S$", if and only if this is true for all measures in $\mathrm{Markov}(G)$.
%\end{definition}
Two DAGs in the same MEC entail the same set of conditional independence relations. \citep{meek1995causal}. %Now, we review the definition of a causal bayesian network from \cite{pearl2009causality}.

%\begin{definition}\label{defn:cause-bayes}
%  Let $P(v)$ be a distribution on variables in the set $V$. Let $P_s(v)$ denote the distribution resulting from the intervention $do(S=s)$ that sets a subset $S$ of variables to constants $x$. Let $\mathbf{P}_{*}$ denote the set of all interventional distributions $P_s(v),~ \forall A \subset V$ including $P(v)$ (when $S= \emptyset$). A directed acyclic graph $G$ is said to be a \textit{Causal Bayesian Network} compatible with the set of distributions $\mathbf{P}_{*}$ if and only if the following three conditions are satisfied by every distribution in $\mathbf{P}_{*}$:
%  \begin{enumerate}[itemsep=-5pt]
 %    \item $P_s(v)$ is Markov with respect to $G$.
  %   \item $P_s(v_i)=1$ for all $V_i \in S$ whenever $v_i$ is consistent with $S=s$.
   %  \item $P_s(v_i \lvert \mathrm{Pa}(v_i)) = P(v_i \lvert \mathrm{Pa}(v_i))$ for all $V_i \notin S$ whenever $\mathrm{Pa}(v_i)$ is consistent with $S=s$.
    %   \end{enumerate}
%\end{definition}

%\subsection{CPDAGs and Observational Markov Equivalence Class}

The MEC of a DAG $G$ can be uniquely represented by a partially directed graph  $\mathrm{Ess}(G)$ known as the \emph{essential graph} of $G$. The skeleton of $\mathrm{Ess}(G)$ is the same as the skeleton of $G$ and the directed edges in $\mathrm{Ess}(G)$ are precisely those edges in $G$ that have the same orientation in all members of the MEC of $G$. All other edges in $\mathrm{Ess}(G)$ are unoriented~\citep{hauser2012characterization}. The following procedure provides all directed edges in $\mathrm{Ess}(G)$:
  \begin{enumerate}[itemsep=-5pt]
   \item For every triple of nodes $i,j,k\in V$  if $i$ and $j$ are disconnected in $G$ and the ordered pairs $(i,k), (j,k)\in E$, then both edges $(i,k)$ and $(j,k)$ are also oriented in $\mathrm{Ess}(G)$. 
  \item Orient edges by successive application of the `Meek rules' (see \citep{meek1995causal} or Appendix \ref{meek}) until they cannot be applied anymore to orient any new edge.
  \end{enumerate}

%Given a directed graph $G$, we describe the notion of a $\mathrm{CPDAG}(G)$ that captures all discoverable causal relationships about $G$ from observational data through conditional independencies. Note that a causal bayesian network satisfies some important invariance properties (Property $3$ in Definition \ref{defn:cause-bayes}) with respect to interventional distributions. However, when only observational distribution $P$ is available, one can characterize $G$ only upto a Markov Equivalent Class. 

%\begin{definition}\label{cpdag}
%  Given a directed graph $G$, a CPDAG consistent with a DAG $G$, denoted by $\mathrm{CPDAG}(G)$ is a partially directed graph with the same skeleton as $G$ where some edges are directed as in $G$ and some of the edges of $G$ are unoriented. The following edges in $G$ are oriented in $\mathrm{CPDAG}(G)$:
%  \begin{enumerate}
%   \item For every edge $(i,j)$  and $(k,j)$ in $G$, and if $i$ and $k$ are disconnected in $G$, then both these edges in $G$ are oriented in $\mathrm{CPDAG}(G)$. 
%   \item Orient edges by successive application of the first three rules in \cite{meek1995causal} (also given in Appendix \ref{meek}) until it cannot be applied anymore to orient any new edges.
%  \end{enumerate}
%\end{definition}

%\begin{theorem}
 %\cite{meek1995causal} All DAGs $G$ that are consistent with a $\mathrm{CPDAG}(G)$ are Markov Equivalent to the causal bayesian network $G$ given the set of conditional independencies from observational distribution $P(v)$.
%\end{theorem}

\subsection{Interventional Markov Equivalence}

Let $I \subset V$ and consider the set of single node interventional distributions $(P_{i})_{i\in I}$, where node $i$ is set to some constant. % for all $i \in {\cal I}$. 
Since in $P_i$, node $X_i$ (a constant) is independent of its parents $X_{\mathrm{Pa}(i)}$, it introduces additional conditional independences in addition to those present in $P$. Let $G^{(i)}$ denote the intervened DAG obtained by deleting the edges from $\mathrm{Pa}(i)$ to $i$. If $P$ is Markov with respect to $G$, then $P_i$ is Markov with respect to $G^{(i)}$. Two DAGs $G$ and $G'$ are in the same \emph{$I$-Markov equivalence class} (I-MEC) if and only if $G^{(i)}$ and $G'^{(i)}$ are in the same MEC for all $i \in I$~\citep{hauser2012characterization}. 

Similarly as in the purely observational setting, an $I$-MEC can be uniquely represented by an \emph{$I$-essential graph} denoted by $\mathrm{Ess}(G,I)$. The skeleton of $\mathrm{Ess}(G, I)$ is the same as the skeleton of $G$ and the directed edges in $\mathrm{Ess}(G,I)$ are precisely those edges in $G$ that have the same orientation in all members of the $I$-MEC of $G$. All other edges in $\mathrm{Ess}(G,I)$ are unoriented. The following procedure provides all directed edges  in $\mathrm{Ess}(G, I)$:
  \begin{enumerate}[itemsep=-4pt]
   \item For every triple of nodes $i,j,k\in V$ with $(i,k), (j,k)\in E$ and if $i$ and $j$ are disconnected in $G$, then both edges $(i,k)$ and $(j,k)$ are also oriented in $\mathrm{Ess}(G)$. 
      \item For every edge $(i,j)$ such that either $ j \in  I$ or $i \in I$, then $(i,j)$ is oriented.
  \item Orient further edges by successive application of the four rules in \citep{hauser2012characterization} (also given in Appendix \ref{meek}) until it cannot be applied anymore to orient any new edges.
  \end{enumerate}

\subsection{Metrics of Interest}

Suppose that the causal Bayesian network that generates data (both interventional and observational) is an orderDAG $G_n$. Let $\mathbf{P}_{*}$ be an associated family of interventional distributions  compatible with $G_n$. In this setting, our work  asymptotically characterizes some metrics that reflect identifiable portions of $G_n$ from an observational distribution $P$ nd possibly also interventional distributions. 

We denote by uEss an essential graph that is also a DAG, i.e., an essential graph representing an MEC consisting of a unique DAG. Such DAGs are of particular interest since they are identifiable from purely observational data. 

In the following, we will measure the degree of identifiability of a random DAG $G_n \sim orderDAG(n, \rho)$ using the following metrics:
\begin{enumerate}[itemsep=-5pt]
 \item Let $X_n$ be the number of unoriented edges in $\mathrm{Ess}(G_n)$. We show that $X_{\infty} := \lim \limits_{n \rightarrow \infty} X_n$ exists.
 \item Let $isuEss_n$ be an indicator variable that is $1$ only if $\mathrm{Ess}(G_n)$ is a DAG. Similarly, the limit is denoted $isuEss_{\infty}$
 \item Let $I_n$ be the number of single node interventions required to fully orient $G_n$. Similarly, the limit is denoted $I_{\infty}$.
 \item Let $L_n$ be the size of the (observational) MEC of $G_n$. The limit is denoted $L_{\infty}$.
 \item Let $X_n(r)$ be the minimum number of unoriented edges in $\mathrm{Ess}(G_n,I)$ optimized over all $I: \lvert I \rvert=r$. The limit is denoted $X_{\infty}(r)$.
 \item Let $isuEss_n(r)$ be an indicator variable that is 1 when $X_n(r) = 0$. The limit is denoted $isuEss_{\infty}(r)$.
 \item Let $L_n(r)$ be the size of the interventional markov equivalence class when the interventions in the set $I$ are performed on $G_n$, where $I$ minimizes the number of unoriented edges in $\mathrm{Ess}(G_n,I)$ optimized over all $I: \lvert I \rvert=r$. This limit is denoted $L_{\infty}(r)$.
\end{enumerate}

\section{Main Results}

We first describe the nature of our results and the approach taken for obtaining these results for $X_n$. The results for all other metrics follow using a similar approach, although the technical details differ depending on the metric of interest. We show that $\E(X_n) \leq \E(X_\infty) \leq \E(X_n) + \epsilon_n$ 
and we provide an explicit expression for $\epsilon_n$. As a consequence, tight upper and lower bounds can be constructed on the quantities of interest by numerically computing $\E[X_n]$ using Monte Carlo simulations by generating random order DAGs $G_n$ for large $n$ and averaging.

Formally, we state the main result in our work about the asymptotic quantities of various metrics.
%
%\begin{lemma}\label{RHS}
%Let $\mathrm{RHS}(\rho,n) = \rho n*(1 - \rho (1 - \rho))^{n-1}$. Then, 
%\begin{align}
%\epsilon_n=\sum_{i \geq n} \mathrm{RHS}(\rho,i) &\leq \frac{(1 - \rho (1 - \rho))^n}{\rho (1-\rho)^2} + \nonumber \\
%\hfill & n\frac{(1 - \rho (1 - \rho))^{n-1}}{(1-\rho)}  
%\end{align}
%\end{lemma}
%
\begin{theorem}\label{asymp}
We have the following inequalities satisfied by various metrics:
 \begin{align}
   E[X_n(r)] &\leq E[X_{\infty}(r)] \leq E[X_n(r)] + \epsilon_n  \nonumber \\
   E[I_n] &\leq E[I_{\infty}] \leq E[I_n] + \epsilon_n \nonumber \\
   E[\log_2 (L_n(r))] &\leq E[\log_2(L_{\infty}(r))] \leq E[X_{\infty}(r)] \nonumber \\
   E[isuEss_n(r)] &\geq E[isuEss_{\infty}(r)] \geq E[isuEss_n(r)] - \epsilon_n  \nonumber 
 \end{align}
 for all $r=0,1,2 \ldots$. Here, $\epsilon_n$ is defined as follows:
 \begin{align}
\epsilon_n=\sum_{i \geq n} \mathrm{RHS}(\rho,i) &\leq \frac{(1 - \rho (1 - \rho))^n}{\rho (1-\rho)^2} + \nonumber \\
\hfill & n\frac{(1 - \rho (1 - \rho))^{n-1}}{(1-\rho)},\label{eqn:eps}
\end{align}
where $\mathrm{RHS}(\rho,n) = \rho n*(1 - \rho (1 - \rho))^{n-1}$ and $\rho$ is the edge probability when sampling an order DAG.
\end{theorem}

We establish the main result on upper and lower bounds through intermediate results as follows (explained taking the example of $X_n$): a) We first exhibit a coupling between $G_n$ and $G_{n+1}$ such that their respective marginal distributions are preserved. This is done in Section \ref{sec:coupling}. b) Using the properties of this specific coupling, we first show that $\E[X_n]$ is a monotonic sequence in $n$ in Section \ref{monotone}. c) The expression for $\epsilon_n$ is obtained by upper bounding the successive differences $\E[X_{n}]-\E[X_{n+1}]$ again using the properties of order DAG sampling and the coupling. This is explained in Sections \ref{obsgap} and \ref{intgap}. Other sections provide additional results on I-MECs obtained through other interventional design algorithms along with numerical and simulation results.

\subsection{Probability coupling}\label{sec:coupling}

In this section, we provide a coupling argument between the distribution of $G_n$ and $G_{n+1}$ such that `un-orientability' properties of certain edges are preserved.

For all $1\leq i < j\leq n$, let $A_{i,j}$ be a binary random variable that is 1 with probability $\rho$. Let $G_n$ be the DAG with nodes $v_1 \ldots v_n$ and directed edges between $v_i\to v_j$ if and only if $A_{i,j} = 1$. %, and such an edge is directed from $v_i$ to $v_j$

\vspace{0.2cm}
\begin{observation}
$G_n$ with permutation $v_1, v_2, \ldots, v_n$, has the distribution of a random orderDAG on $n$ vertices with density $\rho$.
\end{observation}
\textbf{Remark:} Observation 1 says that randomly sampling a symmetric adjacency matrix (undirected graph with edge probability $\rho$), permuting rows and columns with a random permutation, and then taking the upper triangular part (orienting the graph according to the permutation) is the same as fixing the permutation from 1,2..n and populating the upper triangular part randomly.

\textbf{Coupling:} Motivated by the above observation, we couple $G_n$ and $G_{n+1}$ as follows. We first generate $A_{i,j}$ for $1\leq i<j \leq n$ as above and use that to orient $G_n$. Then, we generate additional random variables $A_{i,n+1}$ for all $1 \leq i\leq n$ and orient the edges incident to $v_{n+1}$ accordingly.

The above coupling along with certain structural properties of Meek Rules (given in Appendix \ref{meek}) leads to the following results on orientability of certain edges in $G_n$ and $G_{n+1}$ under the coupling.
%Let $X_n$ be the number of unorientable edges of $G_n$

%Let $X(r)_n$ be the minimum number of unorientable edges of $G_n$ after $r$ internvetions (the minimum is over all possible sets of interventions).

%Let $I_n$ be the minimum number of single node interventions needed to orient all edges of $G_n$

%Let $isPDAG_n = 1$ if $X_n = 0$, and 0 otherwise. Observe that $E(isPDAG_n)$ is the probability that $G_n$ is a uPDAG.

%Let $isPDAG(r)_n = 1$ if $X(r)_n = 0$, and 0 otherwise.

%We can similarly define $X_n(C)$ and $isPDAG_n(C)$ to be the number of undirected edges and whether $G_n$ is a uPDAG after $C$ interventions, and similar results will hold. 

%Let $L_n$ be the size of equivalence class containing $G_n$

\vspace{0.2cm}
\begin{lemma}\label{orderUndirected}
Under the above coupling, if an edge $(i,j)$ is unorientable in $G_n$, it is also unorientable in $G_{n+1}$.
\end{lemma}

\begin{lemma}\label{orderRUndirected}
Under the above coupling, if after a set  of interventions $R$ on $G_n$ the edge $(i,j)$ is unorientable in $G_n$, then it is also unorientable in $G_{n+1}$ after the same set of interventions on $G_n$ together with an intervention on $v_{n+1}$.
\end{lemma}

\subsubsection{Monotonicity Lemmas}
\label{monotone}
%\begin{lemma}\label{Xmonotone}
%$X_{n+1} \geq X_n$.
%\end{lemma}
%\begin{proof}
%This follows directly from Lemma %\ref{orderUndirected}.
%\end{proof}

We prove that expected values of all metrics of interest are monotonic in $n$ using the properties of the coupling demonstrated above. First, we show this for observational quantities by appealing to Lemma \ref{orderUndirected}.

\begin{theorem}\label{ELmonotone}
The following statements hold with probability $1$ for the coupling between $G_n$ and $G_{n+1}$: \\
a) $X_{n+1} \geq X_n$. \\
b) $L_{n+1} \geq L_n$. \\
c) $I_{n+1} \geq I_n$. \\
Therefore, $\E(X_{n+1}) \geq \E(X_n)$, $\E(L_{n+1}) \geq \E(L_n)$ and $\E(I_{n+1}) \geq \E(I_n)$.
\end{theorem}

%\begin{theorem}\label{ELmonotone}
%$L_{n+1} \geq L_n$ with probability $1$ under %the coupling between $G_n$ and $G_{n+1}$. Hence, %$\E(L_{n+1}) \geq \E(L_n)$.

%\end{theorem}

%\begin{theorem}\label{EImonotone}
%$I_{n+1} \geq I_n$ with probability $1$ %according to the coupling between $G_n$ and %$G_{n+1}$. Hence, $\E(I_{n+1}) \geq \E(I_n)$.
%\end{theorem}

%\begin{theorem}\label{EImonotone}
%$\E(I_{n+1}) \geq \E(I_n)$.
%\end{theorem}
%\begin{proof}
%This follows directly from Lemma %\ref{Imonotone}.
%\end{proof}

%\begin{lemma}\label{isPDAGmonotone}
%$isPDAG_{n+1} \leq isPDAG_n$.
%\end{lemma}
%\begin{proof}
%If $G_{n+1}$ is uEss, then by Lemma %\ref{orderUndirected}, $G_n$ is also %uEss.
%\end{proof}

%\begin{theorem}\label{EisPDAGmonotone%}
%$\E(isuEss_{n+1}) \leq \E(isuEss_n)$
%\end{theorem}
%\begin{proof}
%This follows directly from Lemma %\ref{isPDAGmonotone}.
%\end{proof}
Similar monotonicity properties for interventional quantities are obtained by appealing to Lemma \ref{orderRUndirected}. However, note that these proofs are not a straightforward application of Lemma \ref{orderRUndirected}. Often, additional arguments need to be made to show the following results.

\begin{theorem}\label{EXRmonotone}
$X_{n+1}(r) \geq X_n(r)$ with probability $1$ according to the coupling between $G_n$ and $G_{n+1}$. Hence, $\E(X_{n+1}(r)) \geq \E(X_n(r))$.
\end{theorem}

The previous two theorems directly provide the following result.

\begin{theorem}\label{EisPDAGRmonotone}
$isuEss_{n+1}(r) \leq isuEss_n(r)$ for all $r=0,1,2 \ldots$ best interventions with probability $1$ under the coupling between $G_n$ and $G_{n+1}$. Hence,
$\E(isuEss_{n+1}(r)) \leq \E(isuEss_n(r))$.
\end{theorem}
\begin{proof}
This follows directly from Theorem \ref{EXRmonotone} and Theorem \ref{ELmonotone}.
\end{proof}

\begin{theorem}\label{ELRmonotone}
$L_{n+1}(r) \geq L_n(r)$ with probability $1$ under the coupling between $G_n$ and $G_{n+1}$. Hence, $\E(L_{n+1}(r)) \geq \E(L_n(r))$.
\end{theorem}

The established monotonicity results help prove that the asymptotic versions of these metrics exist.

%\begin{theorem}\label{ELmonotone}

%\end{theorem}
%\begin{proof}
%This follows directly from Lemma %\ref{Lmonotone}.
%\end{proof}

\begin{theorem} \label{limexists}
$\lim \limits_{n \rightarrow \infty} X_n = X_{\infty}$ exists and $\E[X_{\infty}]=\lim_{n \rightarrow \infty} \mathbb{E}[X_n]$.
\end{theorem}

\textbf{Remark:} Theorem \ref{limexists} extends to all  metrics that have been shown to be monotonic non-decreasing, i.e. metrics in the set $\{X_n(r),I_n,L_n(r)\}$, by analogous arguments. Note that monotonically non-increasing sequences like $isuEss_n(r)$ are bounded below and above and hence the results can be shown again by the same theorem applied to shifted negatives of these variables.

\subsubsection{Gap Bounds on Observational Metrics}\label{obsgap}

Using properties of the coupling between $G_n$ and $G_{n+1}$ we can show that the expected  difference in the observational metrics for $G_n$ and the asymptotic version is bounded.

%Let $E(X_\infty)$ be $lim_{n -> \infty} X_n$, if such limit exists.

%\begin{lemma}\label{lazyXbound}
%$\E(X_{n+1}) - \E(X_n) \leq \rho n*(1 - \rho (1 - \rho))^{n-1}$.
%(we define $RHS = \rho n*(1 - \rho (1 - \rho))^n)$)
%\end{lemma}

\begin{theorem}\label{lazyEXbound}
$\E(X_\infty) - \E(X_n) \leq \sum^\infty_{i = n} \rho i*(1 - \rho (1 - \rho))^{i-1}$.
\end{theorem}

\begin{theorem}\label{lazyEIbound}
$\E(I_\infty) - \E(I_n) \leq \sum^\infty_{i = n} \rho i*(1 - \rho (1 - \rho))^{i-1}$.
\end{theorem}

\subsubsection{Gap Bounds on Interventional Metrics} \label{intgap}
In the following, we show that the expected  difference in the interventional metrics for $G_n$ and the asymptotic version is bounded again using the properties of the coupling described before.

%\begin{lemma}\label{lazyXRbound}
%$\E(X_{n+1}(r)) - \E(X_n(r)) \leq \rho n*(1 - \rho (1 - \rho))^{n-1}$.
%\end{lemma}

\begin{theorem}\label{lazyEXRbound}
$\E(X_\infty(r)) - \E(X_n)(r) \leq \sum^\infty_{i = n} \rho i*(1 - \rho (1 - \rho))^{i-1}$.
\end{theorem}
%\begin{proof}
%This follows from Lemma %\ref{lazyXRbound}
%\end{proof}

\begin{theorem}\label{lazyEisPDAGbound}
$\E(isuEss_n(r)) - \E(isuEss_\infty(r)) \leq \sum^\infty_{i = n} \rho i*(1 - \rho (1 - \rho))^{i-1}$.
\end{theorem}

%\begin{lemma}\label{lazyLbound}
%(unsure if provable with current approach, would need to change right hand side even if it is) $(L_{n+1})/E(L_n) \leq \sum^\infty_{i = n+1} \rho i*(1 - \rho (1 - \rho))^i$
%\end{lemma}
%
%\begin{theorem}\label{lazyELbound}
%(unsure if provable with current approach, would need to change right hand side even if it is)  $E(L_\infty)/E(L_n) \leq \sum^\infty_{i = n+1} \rho i*(1 - \rho (1 - \rho))^i$
%\end{theorem}

All these results together allow us to prove the main result (Theorem \ref{asymp}).

%\subsubsection{Proof of the Main Result}

\begin{proof}[Proof of Theorem \ref{asymp}]
The theorem follows from results in Sections \ref{monotone}, \ref{obsgap}, and \ref{intgap}. We use the fact that $\log_2(L_n(r)) \leq X_n(r)$, since $L_n(r) \leq 2^{X_n(r)}$ by considering all possible orientations of the unoriented edges in the $I$-essential graph.
\end{proof}

\subsubsection{Lower Bound on Successive Differences}

The above gap bounds depend on upper bounding successive differences of $\E[X_n]$. In the following, we provide a lower bound on the successive differences which implies that gap bounds that are faster than exponential cannot exist.  

\begin{theorem}\label{XisSlow}
$$\E(X_{n}) - \E(X_{n-1}) \geq (n-1) \rho (1-\rho)^{2n - 4} \geq \rho (1-\rho)^{2n}.$$
\end{theorem}

\section{Results on I-MECs obtained by Interventional Design Algorithms}

In the following, we provide asymptotic convergence rates for the number of undirected edges after $r$ interventions, when the interventions are chosen by an algorithm that has a property that we call \textit{downstream-independence}. Greedy algorithms that choose $r$ interventions sequentially based on the essential graph at the observational stage are downstream-independent. Note that, in this section, we do not consider $X_n(r)$, which is the minimum number of edges left unoriented when $r$ interventions are chosen based on the DAG structure. We are therefore interested in algorithms that optimize the interventions based on the essential graph, which can be inferred from purely observed datasets.

\begin{notation}
  Let $J$ be a set of interventions. We say that $H = J(G)$ when $H$ is the essential graph that results from performing the interventions $J$ on the underlying causal DAG $G$. Note that if $G'$ is a subgraph of $G$, then $J(G')$ is obtained by skipping the interventions on nodes outside of $G'$.
\end{notation}  
  
\begin{lemma}\label{oneVertexAlg}
Let $G$ be a DAG and $v_n$ a vertex of $\mathrm{ess}(G)$ with no outgoing or undirected edges. Then, $J(G \backslash v_n) = J(G)\backslash v_n$. In other words, interventions do not affect vertices that have no outgoing or undirected edges.
\end{lemma}

\begin{lemma}\label{lem-subgraph-ignore}
Let $G'$ be an induced subgraph of $G$ consisting of all vertices $v_i$ such that neither $v_i$ nor any descendants of $v_i$ have adjacent undirected edges. Then $J(G \backslash G') = J(G)\backslash G'$.
\end{lemma}
\begin{proof}
The proof follows by applying Lemma \ref{oneVertexAlg} recursively to $G$.
\end{proof}

\begin{definition}
We say that an algorithm $A$ for performing interventions on an essential graph is \textbf{downstream-independent} if the inverventions it performs on $G$ are identical to the ones it performs on $G \backslash G'$. 

Note that $G \backslash G'$ is the result of the following process: starting with $G$, recursively remove vertices that have no undirected or outgoing edges.
\end{definition}

\begin{theorem}\label{downstreamConvergence}
Let $A$ be a downstream-independent algorithm. Let $Y(r,A)_i$ be the expected number of undirected edges in the essential graph of the random order DAG $G_i$ after performing $r$ interventions according to algorithm $A$. Then
\begin{align}
\lvert \E(Y(r,A)_{i+1} - \E(Y(r,A)_{i}) \rvert &\leq \rho i*(1 - \rho (1 - \rho))^{i-1}  \nonumber \\
\hfill & *i(i+1)/2
\end{align}
\end{theorem}

\textbf{Remark:}  Suppose there is an algorithm $A$ that optimizes some score function based on the essential graphs alone which is a proxy for minimizing the number of expected unoriented edges after $r$ interventions, then such algorithms are likely to be making decisions independent of $G'$ in general due to Lemma \ref{lem-subgraph-ignore}. An example is the algorithm that greedily picks the intervention that reduces the expected number of unoriented edges where the expectation is over the uniform distribution of DAGs compatible with the essential graph.

%(note: this bound is fairly loose, if we need to tighten it, we can, but is it important enough given that the exponent will remain the same?)

%\begin{theorem}\label{alg-downstream}
%There exists an algorithm that minimizes the number of undirected edges after $k$ interventions that is downstream-independent. %(i.e. there is a downstream-independent version of OPT).
%\end{theorem}

\begin{theorem}\label{algoasymp}
Let $A$ be an algorithm that is downstream independent and chooses interventions based on $\mathrm{ess}(G)$. Let $Y(r,A)_n$ be the number of undirected edges after $r$ interventions made by the algorithm $A$. Then, 
\begin{align*}
 \E[Y(r,A)_n] &\leq \E[Y(r,A)_{\infty}] \nonumber \\
 \hfill &\leq \E[Y(r,A)_n] + \nonumber \\
 \hfil & \quad\sum_{i=n}^{\infty} \rho i^2(i+1)/2 * (1-\rho(1-\rho))^{i-1}.
\end{align*}
Here, $\lim \limits_{n \rightarrow \infty} Y(r,A)_n = Y(r,A)_{\infty}$ and this limit exists.
\end{theorem}
\begin{proof}
This is a direct corollary from the previous results in this section together with analogous arguments regarding monotonicity and existence of limits similar to those for $X_n(r)$.
\end{proof}

\begin{figure*}[htbp]
    \subfloat[Average number of unoriented edges, $X_n(r)$, in the essential graph associated with orderDAGs of density $\rho$ after $r$ interventions, averaged over $2000$ samples; the highlighted region corresponds to points within 2-standard deviations from the mean. \label{fig1}]
    {\includegraphics[width=8cm]{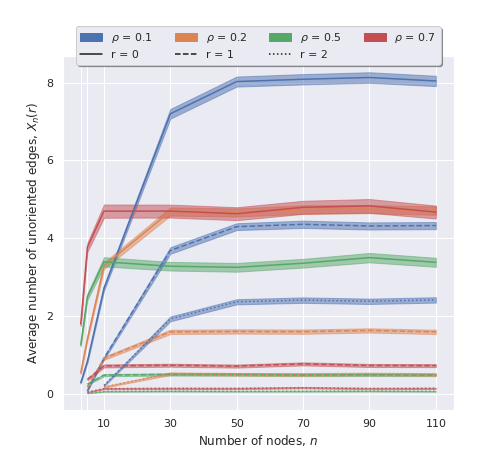}}
    \hfill
    \subfloat[Average logarithm of the size of the I-MEC for order DAGs of density $\rho$ after $r$ interventions, averaged over $2000$ samples; the highlighted region corresponds to points within 2-standard deviations from the mean. \label{fig3}]{\includegraphics[width=8cm]{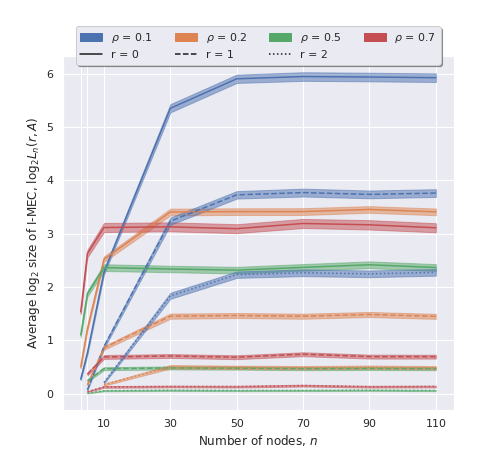}}

    \caption{We plot Monte-Carlo estimates of $\E[Y_n(r,A)]$, i.e~ the number of unoriented edges in the essential graph of a random order DAG after $r$ interventions, together with $\E[\log_{2} L_n(r,A)]$, i.e. the size of the $I$-MEC after $r$ interventions.
    }
\end{figure*}

\section{Discussion of the Results}
Theorems \ref{asymp} and \ref{algoasymp} provide upper bounds in terms of quantities computable by Monte-Carlo simulation at finite $n$ from random order DAGs and constants such as $\epsilon_n$ that are exponentially small in $n$. If empirical means of these finite $n$ quantities appearing in these upper bounds can be characterized with very high precision, then we can characterize the constant by which these asymptotic quantities are upper bounded.

In the following section, we plot the empirical means of these finite $n$ quantities or upper bounds to these finite $n$ quantities for very large $n$ and show that when combined with the above bounds, the asymptotic quantities tend to a constant.

\subsection{Precise Calculation of High Confidence Upper Bounds on Asymptotic $\log$-MEC Size for Random Order DAGs of Density $\rho=0.5$ }

We demonstrate how to obtain confidence intervals on the expected asymptotic mean $E[X_\infty]$ and $E[\log_2 (L_\infty)]$ using our bounds and Monte Carlo simulations.

\textbf{Details of the Numerical Experiment:} We sampled $X_{30}$ $S=100000$ times for random order DAGs with $\rho=0.5$. The sample variance we observed was $V=7.054$ while the empirical mean was $M=3.394$.

We use an empirical Bernstein bound for $E[X_{30}]$ and show the following bound on expected value of $X_{\infty}$: 
\begin{theorem}
\label{thm-upp-bounds}
 With probability at least $0.99$ over the randomness in our numerical experiments over $S=100000$ samples, we have: 
  $E[\log_2 (L_{\infty})] \leq E[X_{\infty}] \leq 3.497 $.
\end{theorem}

This is an illustration of how our upper bounds, emprirical Bernstein bounds and Monte Carlo simulation can be combined to give highly precise guarantees for all the considered metrics.

\section{Numerical Results}\label{numerical-sec}

We compute and plot the empirical means of the following observational metrics: a) $X_n$, b) $isuEss_n$, c) $I_n$, and d) $\log_{2} L_n$. We also plot the empirical mean of the following interventional metrics a) $Y(r,A)_n$,  b) $isuEss(r,A)_{n}$, c) $\log_{2} L(r,A)_{n}$, and d) $I(r,A)_n$. These interventional metrics are obtained on the essential graph $\mathrm{Ess}(G_n,A)$ obtained by the  greedy algorithm $A$ that operates as follows:  First pick the node $I_1$ that orients the most edges, then for each consecutive $r$, pick $I_r$ that orients the most edges in $G_n$ given the ($\{ I_1, \ldots, I_{r-1} \}$-)essential graph.

\begin{figure*}[htbp]
    \subfloat[Probability that the essential graph associated with an order DAG of density $\rho$ can be uniquely identified after $r$ interventions, averaged over $2000$ samples; the highlighted region corresponds to points within 2-standard deviations from the mean. \label{fig2}]{\includegraphics[width=8cm]{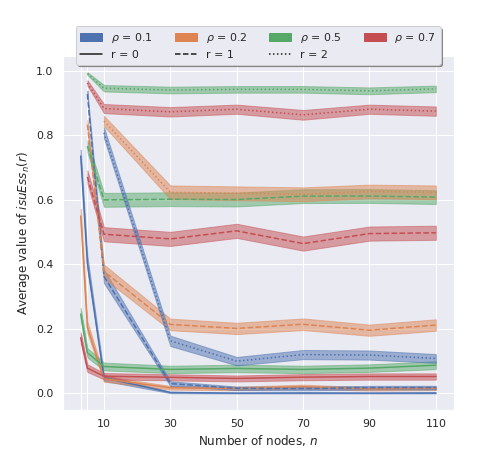}}
    \hfill
    \subfloat[Empirical mean of the number of interventions needed to fully identify a random order DAG of density $\rho$, averaged over $2000$ samples; the highlighted region corresponds to points within 2-standard deviations from the mean.  \label{fig4}]{\includegraphics[width=8cm]{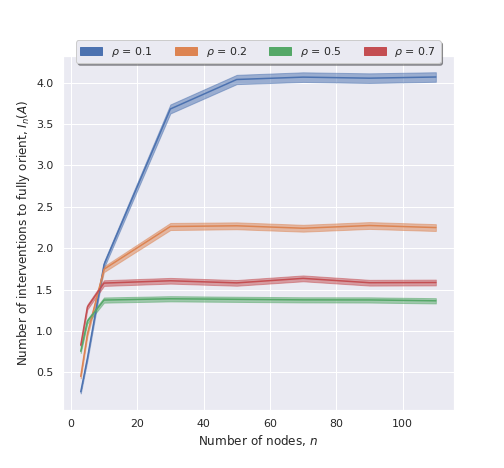}}
    
    \caption{We plot Monte-Carlo estimates of $\mathbb{P}(isuEss_n(r,A))$, i.e.~the probability that the essential graph of a random order DAG is equal to the order DAG itself, together with $\E[I_n(r,A)]$, i.e.~the number of single-node interventions required to fully orient a random order DAG.}
\end{figure*}

\textbf{Graph Generation:}
We generated 2,000 random order DAGs with $n = \{3, 5, 10, 30, \ldots 110 \}$ nodes and densities $\rho = \{ .1, .2, .5, .7\}$. For each DAG, we used the open-source \texttt{causaldag} package in Python to compute the number of DAGs in the ($\mathcal{I}$-)MEC and the number of undirected edges in the ($\mathcal{I}$-)essential graph obtained by applying algorithm $A$ on $G_n$.

\textbf{Results Established:}
The plots serve two purposes - a) The empirical mean plots (Figs. \ref{fig1}-\ref{fig4}) and the box plots (Figs. \ref{fig7}-\ref{fig12})  of all the estimated quantities provide an idea of what values the asymptotic quantities are bounded by given the formula for $\epsilon_n$ in Theorem \ref{asymp}. For a more refined high confidence upper bound, for large enough $n$, analysis similar to Theorem~\ref{thm-upp-bounds} can be done. b) They help corroborate the monotonicity results we have derived analytically.

\textbf{Bounding Interventional Metrics:} We observe that the above interventional metrics plotted provide an upper bound to $X_n(r),L_n(r),isuEss_n(r)$ and $I_n(r)$ which are based on the set of optimal interventions for $G_n$ that minimize the number of unoreinted edges given $G_n$. Therefore, by Theorem \ref{asymp} they certainly provide valid upper bounds together with $\epsilon_n$. The shaded regions in each plot are the estimates of the 95\% confidence intervals as given by the \texttt{scipy.stats} function \texttt{bayes\_mvs}.

Figure \ref{fig1} plots empirical mean of $X_n$ and $Y(r,A)_n$. We observe that $\bar{X}_n$ increases sharply for $\rho \geq 0.5$ and plateaus near $n = 10$, while $\bar{X}_n$ increases more gradually for $\rho < 0.5$, with a higher limit for sparse graphs. For all densities, the empirical mean of $Y(r,A)_n$ increases more gradually than the observational $\bar{X}_n$.

Figure \ref{fig3} plots empirical mean of $\log L_n$ and $\log L(r,A)_n$. We again observe sharper increases and lower plateaus for the higher densities, $\rho = 0.5$ and $\rho = 0.7$, compared to more gradual rises and higher plateaus for the lower densities. Whereas in Figure~\ref{fig1}, $\bar{X}_n$ stabilizes at similar values for $\rho = 0.2$ and $\rho = 0.5$, in Figure~\ref{fig3}, the empirical mean of $\log L_n$ is greater for $\rho = 0.2$ than for $\rho = 0.5$. This indicates that each unoriented edge contributes to more MECs when the density is low.

Figure \ref{fig2} demonstrates the monotonicity of the empirical mean of $isuEss_n$ and $isuEss(r,A)_n$. We observe that the empirical mean of $isuEss_n$ drops sharply for all densities, with $\rho = 0.5$ appearing to have the highest limit. The difference in behavior of the empirical mean of $isuEss(1,A)_n$ and $isuEss(2,A)_n$ for different densities is noteworthy. For sparser graphs, 1 or 2 interventions do not significantly increase the expected ability to identify the DAG; for instance, when $\rho = 0.1$, the expected number of fully identified DAGs barely changes from the observational case after $n = 30$. However, for denser graphs, such as for $\rho = 0.5$ and $\rho = 0.7$, even 1 intervention is sufficient to learn roughly 50\% and 60\% of the sampled graphs, respectively, and 2 interventions is sufficient to learn nearly all of them, even when $n = 110$. This result can be explained by the fact that sparse graphs often consist of multiple connected components and interventions in one component have no effect on other components. Finally, Figure~\ref{fig4} demonstrates the monotonicity of the empirical mean of $I_n$. Surprisingly, it takes very few interventions to orient even~large,~sparse~graphs.

% All figures are based on 100,000 samples for each value of $\rho$ in order to obtain tight bounds. %, we took 100,000 samples for $n = 100$ for each $\rho$.

\section{Conclusion}
We provided sharp upper and lower bounds for asymptotic expected $\log$-MEC size and the number of interventions needed to fully orient a random order DAG after $r=0,1,2..$ (constant) number of initial interventions. There are various other metrics associated with $I$-MECs of random order DAGs that we precisely quantify in this work. Our methods relied on analytical bounds on the asymptotic quantities based on coupling arguments and exploiting the properties of Meek rules. This together with Monte Carlo simulations at finite sizes establishes quantifiable and precise bounds. 
 
Our results mean that a walk over the space of graphs (larger search space but simpler moves) would not be more time consuming than a walk over the space of Markov equivalence classes (more complicated moves) when implementing greedy search for structure learning. This is because the asymptotic log MEC size goes to a constant for dense graphs. %In terms of implications on algorithm design, this result indicates (at least for relatively dense graphs) that it is not worth using the more complicated walks on Markov equivalence classes.
 In addition, our results imply that in general relatively few interventions are needed to identifying dense causal networks. Investigations like this for random graphs considering various levels of sparsity and relaxing the causal sufficiency assumptions are interesting directions for future work.

\section*{Acknowledgements}
\vspace{-0.4cm}
C.~Uhler was partially supported by NSF (DMS-1651995), ONR (N00014-17-1-2147 and N00014-18-1-2765), IBM, and a Sloan Fellowship.

%\newpage
\bibliographystyle{abbrvnat}
\bibliography{refs-order.bib}

\newpage
\quad
\appendix

\section{Meek Orientation Rules}\label{meek}

In in Figure~\ref{fig5}, we provide the four Meek orientation rules that are used in the definition of the essential graph.
\begin{figure}[htbp]
 \centering
 \includegraphics[width=8.5cm]{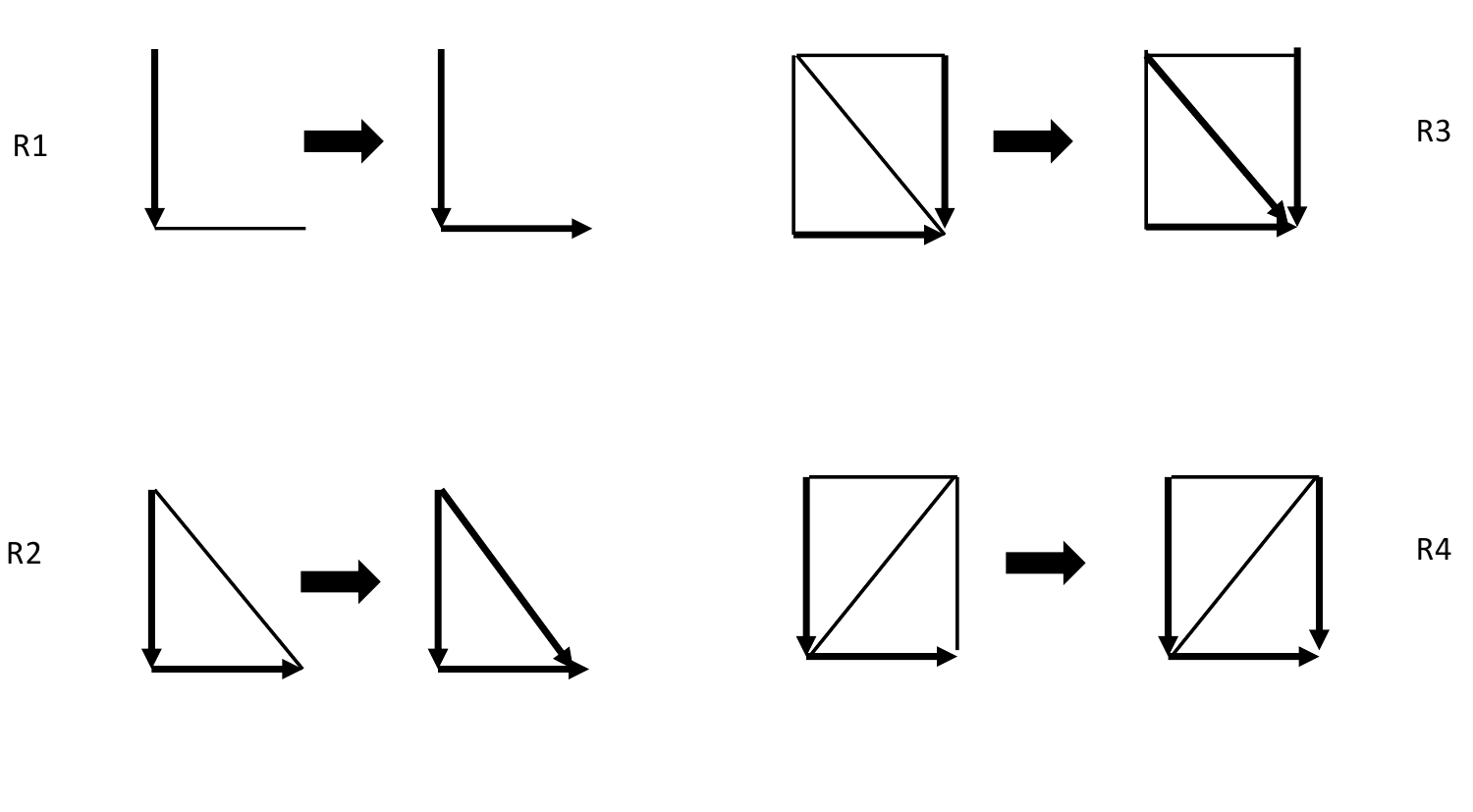}
 \caption{Meek orientation rules used to direct edges in the interventional essential graph representing the $I$-MEC.}
 \label{fig5}
\end{figure}

The following two observable properties play an important role in various results in the main paper. 

\begin{property}\label{prop1}
 If a node $v$ is involved in any of the four Meek rules and if the node $v$ does not have an outgoing edge in the original causal DAG, then the oriented edge (in the right hand side motif of  any of the four rules in Figure \ref{fig5}) is incident to $v$.
\end{property}

\begin{property}\label{prop2}
 If a node $v$ is involved in a motif for any of the four rules, then either $v$ has an  outgoing edge or it has an adjacent undirected edge (on the left hand side motif appearing in that rule).
\end{property}

\section{Additional Proofs}

\subsection{Proof of Lemma \ref{orderUndirected}}

Observe that all edges between $G_n$ and $v_{n+1}$ are directed to $v_{n+1}$ and that $v_{n+1}$ does not have any outgoing edges. Suppose that $v_{n+1}$ is involved in one of the four Meek rules in Appendix \ref{meek}. Then by Property \ref{prop1} in Appendix \ref{meek}, the discovered edge has to be incident to $v_{n+1}$. On the other hand, if $v_{n+1}$ is not part of any Meek rule, then the rules must have already been applied in $G_n$ to orient edges maximally, which completes the proof.

\subsection{Proof of Lemma \ref{orderRUndirected}}

The proof is similar to the proof of Lemma \ref{orderUndirected}, i.e. it follows from Property \ref{prop1} in Appendix \ref{meek}.

%\subsection{Proof of Lemma \ref{lazyXbound} }
%Observe that the left hand side equals the expected number of unorientable edges incident to $v_{n+1}$ in $G_{n+1}$ by Lemma \ref{orderUndirected}. We will upper bound this number as follows:

%For each vertex i the edge $(i,n+1)$ is unoriented if it is present (probability $\rho$), and not part of an uncovered collider. The probability that $(i,n+1)$ and $(j,n+1)$ form an uncovered collider given $(i,n+1)$ is present is $\rho(1-\rho)$, and such probabilities for different $j$ are independent. Thus the probability that $(i,n+1)$ is not part of an uncovered collider given that it is present is $(1 - \rho(1-\rho))^{n-1}$. Multiplying by $\rho$ for the probability that $(i,n+1)$ is present, and by $n$ for the total number of such potential edges leads to the desired bound.

\subsection{Proof of Equation \ref{eqn:eps}}

We can simplify this sum as follows: %(someone please double check my math if this will be the final form, but until then only exponential convergence matters, which should be obvious):
%Observe that:
\begin{align*} 
\sum_{i \geq n} i x^{i-1} &= \frac{d}{dx} (\sum_{i \geq n} x^i) \nonumber \\
 \hfill & = \frac{d}{dx} \frac{x^n}{1-x} \nonumber \\
 \hfill &= \frac{x^n}{(1-x)^2} + n x^{n-1} \frac{1}{1-x} \nonumber 
\end{align*}
Substituting $(1- \rho(1-\rho))$ for $x$, we obtain
\begin{align*}
\sum^\infty_{i = n} \rho i*(1 - \rho (1 - \rho))^{i-1}  &= \rho \left [ \frac{(1 - \rho (1 - \rho))^n}{\rho^2 (1-\rho)^2} +  \right. \nonumber \\
\hfill & \left. n\frac{(1 - \rho (1 - \rho))^{n-1}}{\rho (1-\rho)}  \right] \nonumber \\
\hfill &=  \frac{(1 - \rho (1 - \rho))^n}{\rho (1-\rho)^2} + \nonumber \\
\hfill & n\frac{(1 - \rho (1 - \rho))^{n-1}}{(1-\rho)}  
\end{align*}

\subsection{Proof of Lemma \ref{oneVertexAlg} }
%Suppose $v_n$ in $G$ has neither any adjacent undirected edge nor any outgoing edges. 
Suppose $v_n \notin J$, then by Property \ref{prop2}, it cannot be part of any of the Meek rules in Appendix \ref{meek} . Therefore, it cannot aid in any of the rule applications after new edges have been discovered by interventions in $J$. This means that removing it before or after applying the Meek rules is irrelevant. Hence $J(G \backslash v_n) = J(G)\backslash v_n$. 
%Therefore, $J(G \backslash v_n)$ will result in the same essential graph where you apply $J$ on $G$ and then remove $v_n$. 
If $v_n \in J$, then the intervention on $v_n$ gives no additional information as all its adjacent edges have already been discovered and hence it is equivalent to using $J\backslash v_n$. Hence, this reduces to the  previous case with $J$ replaced by $J \backslash v_n$ and thereby completes the proof.

\subsection{Proof of Theorem \ref{ELmonotone}}

a) This follows directly from Lemma \ref{orderUndirected}. \\
b) Suppose the MEC of $G_n$ is given by the DAG set $\{H_1, H_2, \ldots, H_k\}$. For each $i$, let $H^\prime_i$ be the DAG $H_i$ extended by adding the vertex $v_{n+1}$, with the same incoming edges as it has in $G_{n+1}$. From Lemma \ref{orderUndirected}, it follows that the DAGs $\{H^\prime_1, H^\prime_2, \ldots, H^\prime_k\}$ are contained in the MEC of $G_{n+1}$. \\
c)  Due to the coupling, if a set of interventions orients $G_{n+1}$, it also orients $G_n$. The result follows.

The results for the expected values follow from the almost sure results.

\subsection{Proof of Theorem \ref{EXRmonotone}}
Let $R$ be the set of interventions that achieves the minimum number $X_{n+1}(r)$ of unoriented edges in $G_{n+1}$ and let $|R| = r$. We apply the same set of interventions to $G_n$ barring the possible intervention on node $v_{n+1}$. By Lemma \ref{orderRUndirected}, all edges unorientable in $G_{n}$ after these `copied' interventions are also unorientable in $G_{n+1}$ even with/without the possible additional intervention on $v_{n+1}$. Since we can (possibly) add the  extra intervention in $G_n$ to bring the total number to $r$, this means that after these r interventions on $G_n$ we have at most $X_{n+1}(r)$ unorientable edges in $G_n$, which completes the proof.

\subsection{Proof of Theorem \ref{ELRmonotone}}
Let $R$ be the set of optimal set of $r$ interventions on $G_{n+1}$ that orient the maximum number of edges in the $\mathrm{Ess}(G_{n+1},R)$. Therefore, $\lvert \mathrm{Ess}(G_{n+1},R) \rvert = L_{n+1}(r)$. Let $R'=R \backslash v_{n+1}$. Suppose the $\mathrm{Ess}(G_n,R')$ is given by the DAG set $\{H_1, H_2, \ldots, H_k\}$. For each $i$, let $H^\prime_i$ be the DAG $H_i$ extended by adding the vertex $v_{n+1}$, with the same incoming edges as in $G_{n+1}$. From Lemma \ref{orderRUndirected}, it follows that the DAGs $\{H^\prime_1, H^\prime_2, \ldots, H^\prime_k\}$ are contained in $\mathrm{Ess}(G_{n+1},R)$. This means, that, $L_{n}(r) \leq \lvert \mathrm{Ess}(G_n,R') \rvert \leq \lvert \mathrm{Ess}(G_{n+1},R) \rvert = L_{n+1}(r)  $

\subsection{Proof of Theorem \ref{limexists}}
 For all monotonic sequences $\{ x_n \}$, $$\lim \inf \{ x_n \} = \lim \sup \{ x_n \} = \lim \{ x_n \}.$$ 
 Further, when a sequence of measurable functions $ \{ f_n \}$ converges pointwise to a function $f$, then $f$ is also measurable. Here, $X_n$ is a measurable function of the random variables $G_n$. Hence, $\E[X_{\infty}]=\lim_{n \rightarrow \infty} \E[X_n]$ follows from the Lebesgue Monotone Convergence Theorem.
 
\begin{subsection}{Proof of Theorem \ref{lazyEXbound}}

We first prove the following Lemma. The theorem follows from the Lemma.

\begin{lemma}\label{lazyXbound}
$\E(X_{n+1}) - \E(X_n) \leq \rho n*(1 - \rho (1 - \rho))^{n-1}$.
%(we define $RHS = \rho n*(1 - \rho (1 - \rho))^n)$)
\end{lemma}
\begin{proof}

Observe that the left hand side equals the expected number of unorientable edges incident to $v_{n+1}$ in $G_{n+1}$ by Lemma \ref{orderUndirected}. We will upper bound this number as follows:

For each vertex i the edge $(i,n+1)$ is unoriented if it is present (probability $\rho$), and not part of an uncovered collider. The probability that $(i,n+1)$ and $(j,n+1)$ form an uncovered collider given $(i,n+1)$ is present is $\rho(1-\rho)$, and such probabilities for different $j$ are independent. Thus the probability that $(i,n+1)$ is not part of an uncovered collider given that it is present is $(1 - \rho(1-\rho))^{n-1}$. Multiplying by $\rho$ for the probability that $(i,n+1)$ is present, and by $n$ for the total number of such potential edges leads to the desired bound.
\end{proof}

\end{subsection}

\begin{subsection}{Proof of Theorem \ref{lazyEIbound}}
Since the unoriented edges incident to $v_{n+1}$ can be oriented with at most one intervention each, we have that $I_{n+1} - I_n \leq X_{n+1} - X_n$. This, combined with Theorem \ref{lazyEXbound} results in the desired bound.
\end{subsection}

\subsection{Proof of Theorem \ref{XisSlow}}
If $A_{i,n} = 1$, and for all $j \neq \{i,n\}$ it holds that $A_{i,j} = A{j,n} = 0$, then the edge $(i,n)$ is isolated and thus unorientable. That happens with probability $\rho (1-\rho)^{n-2} (1 - \rho)^{n-2}$ for each $i$. Note that there are $n-1$ such potential edges that are adjacent to vertex $v_n$ and therefore figure into $\E(X_{n}) - \E(X_{n-1})$, which completes the proof.

\subsection{Proof of Theorem \ref{lazyEXRbound}}

We provide a lemma and its proof regarding successive differences of the interventional metric $X_n(r)$. The result in the theorem follows immediately from this.

\begin{lemma}\label{lazyXRbound}
$\E(X_{n+1}(r)) - \E(X_n(r)) \leq \rho n*(1 - \rho (1 - \rho))^{n-1}$.
\end{lemma}
\begin{proof}
Let $X'_{n+1}(r)$ be the number of unoriented edges in $G_{n+1}$ after we apply $r$ interventions that achieve $X_n(r)$ unoriented edges in $G_n$. Observe that $X_{n+1}(r) \leq X_{n+1}'(r)$. Since we can show that $\E(X'_{n+1}(r)) - \E(X_n(r)) \leq \rho n*(1 - \rho (1 - \rho))^{n-1}$ by following the proof of Lemma \ref{lazyXbound} in the proof of Theorem \ref{lazyEXbound} , this completes the proof.
\end{proof}

\subsection{Proof of Theorem \ref{lazyEisPDAGbound}}Observe that $isuEss_n(r) - isuEss_{n+1}(r) \leq X_{n+1}(r) - X_n(r)$, because $isuEss_{n}(r) - isuEss_{n+1}(r)$ is either 0 or 1 (it cannot be -1 by Theorem \ref{EisPDAGRmonotone}), and can only be 1 if $X_{n+1}(r) > X_n(r)$. This, combined with Lemma \ref{lazyXRbound} means that $\E(isuEss_n(r)) - \E(isuEss_{n+1}(r)) \leq \rho i*(1 - \rho (1 - \rho))^{n-1}$, which provides the necessary bound.

\subsection{Proof of Theorem \ref{downstreamConvergence}}
It follows from the proof of Lemma \ref{lazyXbound} that the probability that $v_{n+1}$ has an undirected edge is less than $\mathrm{RHS}(\rho,n)$. If $v_{n+1}$ does not have an undirected edge, then, by
%by Lemma \ref{oneVertexAlg} and 
the fact that $A$ is a downstream independent algorithm it follows that $E(Y(r,A)_{n+1} = \E(Y(r,A)_{n})$, and therefore that $\E(Y(r,A)_{n+1} - \E(Y(r,A)_{n})$ = 0. If $v_n$ is adjacent to undirected edges, then $\E(Y(r,A)_{n+1} - \E(Y(r,A)_{n}) < n(n+1)/2$, the number of possible edges in $G_{n+1}$. The bound follows since this happens with probability $< \rho n*(1 - \rho (1 - \rho))^{n-1})$.

\subsection{Proof of Theorem \ref{thm-upp-bounds}}
We consider the following bound for $\rho=0.5$ from Theorem \ref{asymp}:
\begin{align}\label{eq-1}
E[L_{\infty}] \leq E[X_{\infty}] \leq E[X_{30}] + \epsilon_{30}.
\end{align}

$\epsilon_{30} < 0.02$ (by direct calculation) based on Lemma \ref{RHS}.

Now, we apply the following theorem from (quoted with appropriate modifications for a random variable taking values in $[0,B]$).

\begin{theorem}
\citep{maurer2009empirical} 
If $Z_1,Z_2 \ldots Z_s$ are i.i.d random variables each bounded in $[0,B]$. Let $M$ be the empirical mean and $V$ be the empirical variance of the samples. Then, with probability $1- \delta$, 
 \begin{align}
    E[Z] \leq M +   \sqrt{\frac{2V\log (2/\delta)}{s}} + B \frac{7 \log (2/\delta)}{3(s-1)} 
 \end{align}
\end{theorem}

Now, for $X_{30}$, $B = 450$. 
Substituting $V=7.054,M=3.394,\delta=0.01$ in the above theorem and using (\ref{eq-1}) we have the bound in the theorem.
\section{Additional Figures}
The additional figures regarding numerical simulations are given in Figures \ref{Fig:4} and \ref{Fig:5}.
\begin{figure*}[!h]
    \subfloat[ \label{fig7}]{\includegraphics[width=.3\textwidth]{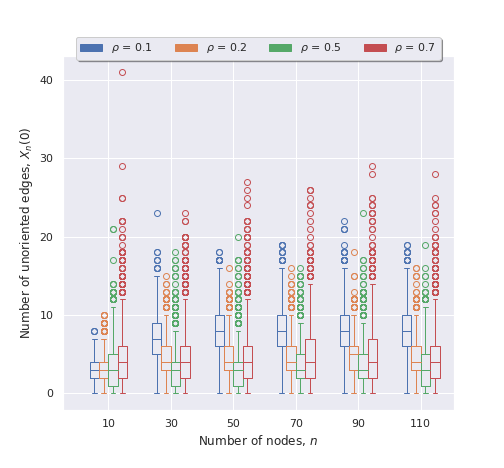}}
    \hfill
    \subfloat[  \label{fig8}]{\includegraphics[width=.3\textwidth]{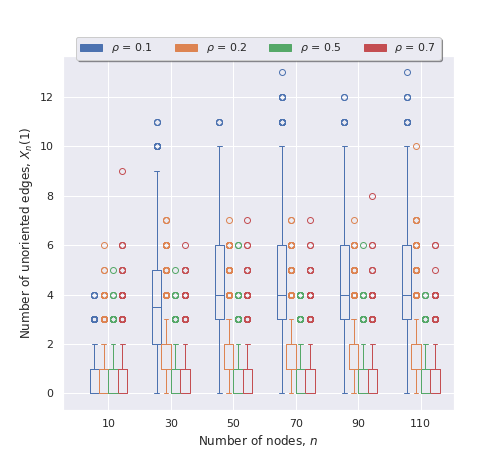}}
    \hfill
    \subfloat[  \label{fig9}]{\includegraphics[width=.3\textwidth]{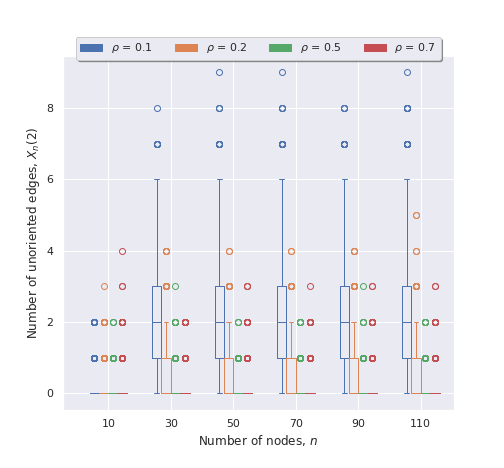}}
    \caption{Number of unoriented edges of the 2,000 orderDAG samples. The middle line of the box is the median, the upper and lower edges are the upper and lower quartiles, and the circles are outliers.}
    \label{Fig:4}
\end{figure*}

\begin{figure*}[!t]
    \subfloat[ \label{fig10}]{\includegraphics[width=.3\textwidth]{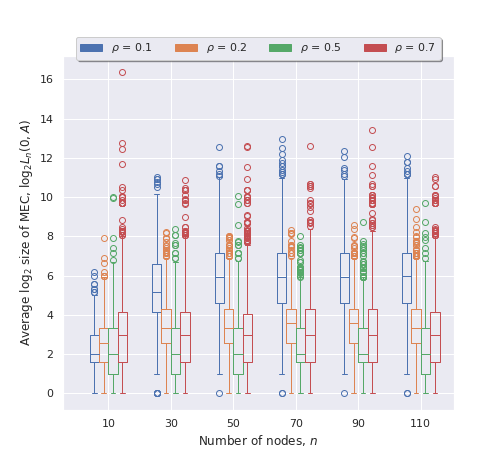}}
    \hfill
    \subfloat[  \label{fig11}]{\includegraphics[width=.3\textwidth]{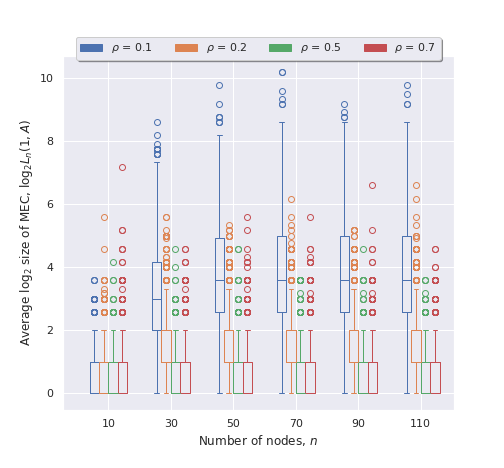}}
    \hfill
    \subfloat[  \label{fig12}]{\includegraphics[width=.3\textwidth]{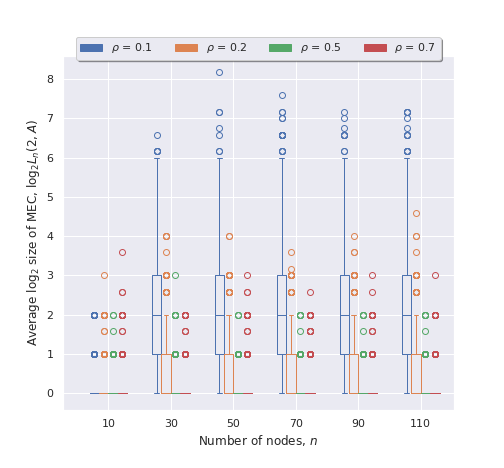}}
    \caption{$\log_2$ MEC sizes of the 2,000 orderDAG samples. The middle line of the box is the median, the upper and lower edges are the upper and lower quartiles, and the circles are outliers.}
     \label{Fig:5}
\end{figure*}

\end{document}